\documentclass[10pt, conference, compsocconf]{IEEEtran}

\usepackage{times}
\usepackage{graphicx}
\usepackage{subfigure} 
\usepackage{algorithm}
\usepackage{algorithmic}
\usepackage{todonotes}
\usepackage{amsmath}
\usepackage{amsfonts}
\usepackage{amsthm}
\usepackage[hyphens]{url}
\usepackage{hyperref}
\usepackage[all]{hypcap}

\DeclareMathOperator{\ReLU}{ReLU}

\newcommand{\vagg}{1.25}
\newcommand{\agg}{1.5}
\newcommand{\nor}{1.75}
\newcommand{\con}{2}
\newcommand{\vcon}{2.5}

\newtheorem{theorem}{Theorem}

\newtheorem{proposition}[theorem]{Proposition}

\begin{document}

\title{Adaptive Neuron Apoptosis for Accelerating Deep Learning on Large Scale Systems}

\author{\IEEEauthorblockN{Charles Siegel, Jeff Daily, Abhinav Vishnu}
\IEEEauthorblockA{
Pacific Northwest National Laboratory\\
Richland, WA 99352\\
charles.siegel@pnnl.gov, jeff.daily@pnnl.gov, abhinav.vishnu.pnnl.gov}
}

\maketitle

\begin{abstract}
We present novel techniques to accelerate the convergence
of Deep Learning algorithms by conducting low overhead removal of
redundant neurons -- {\em apoptosis} of neurons -- which do not
contribute to model learning, during the training phase itself. We
provide in-depth theoretical underpinnings of our heuristics (bounding
accuracy loss and handling apoptosis of several neuron types), and
present the methods to conduct adaptive neuron apoptosis.
Specifically, we are able to improve the training time for several datasets by 2-3x,
while reducing the number of parameters by up to 30x (4-5x on average) on
datasets such as ImageNet classification.  For the Higgs Boson dataset,
our implementation improves the accuracy (measured by Area Under Curve
(AUC)) for classification from 0.88/1 to 0.94/1, while reducing the
number of parameters by 3x in comparison to existing literature. The
proposed methods
achieve a 2.44x speedup in comparison to the default (no apoptosis)
algorithm.
\end{abstract}

\IEEEpeerreviewmaketitle

\section{Introduction}
Deep Learning algorithms emulate computation structure of a brain by
learning models using {\em neurons} and their interconnections ({\em
synapses, also known as parameters/weights})~\cite{dean:nips12}. 
Using a cascade of neurons, Deep Learning algorithms are known to learn
complex non-linear functions.  These functions can be applied to both
{\em supervised} (input dataset with ground truth labels) and {\em unsupervised} (input data with no labels)
problems.
Naturally, Deep Learning algorithms
are being applied to several domains including Computer
Vision~\cite{imagenet}, Speech
Recognition~\cite{graves:nn05}, and High Energy Physics~\cite{sadowski:nips14}.

An important aspect of Deep Learning algorithms is the topology of a
{\em Neural Network} (used interchangeably with Deep Learning with rest
of the paper). A candidate topology may have a single input and an
output layer, with possibly several {\em hidden layers}. Convolutional
Neural Networks (CNN) -- a class of Deep Learning algorithms -- may have
several {\em convolutional layers}, followed by several {\em
fully-connected layers}. In practice, the neurons and synapses are
implemented by using matrices, where each row/column represents a neuron
and each element represents the strength (weight) of a synapse. The
output of a neural network is the weight matrices, which may be used for
Machine Learning tasks such as classification, or clustering. 

Usually, a neural network topology is user-specified, which includes the
number of hidden layers and number of neurons in each layer (an example
is shown in Figure~\ref{fig:example_dnn} (a)). {\em Deeper neural
networks} (with more layers)  are used for model generation from
increasingly complex datasets, possibly learning complex non-linear
functions. Bigger networks -- which have larger number of neurons per
layer -- may also be used in addition to deeper networks for this
purpose. 

However, deeper and/or bigger networks do not necessarily provide better
models.  With increasing number of network {\em parameters} (the overall
number of synapses), the training time per epoch increases
significantly~\cite{hinton:corr15,2016arXiv160302339V}. Deeper networks
tend to suffer from problems such as vanishing
gradients~\cite{vanishinggradient}, where the weights of network
parameters change slowly. In addition, deeper networks are known to
cause {\em overfitting} -- a scenario in which the model has learned very
well from the training set, but does not generalize well on new samples.
This scenario generally occurs when a neural network has learned a
significantly more complex function than implied by the training set.  A
few possible solutions such as Dropout~\cite{dropout} exist -- but they
do not reduce the overall time to solution. Lastly, larger number of
parameters have prohibitive storage and computational requirements
during the testing phase (when the model is applied on new data) --
which is problematic for deployment on power/memory constrained devices. 

\begin{figure*}[htbp]
		\begin{minipage}[t]{2.0\columnwidth}
				\centering
\includegraphics[width=0.7\columnwidth]{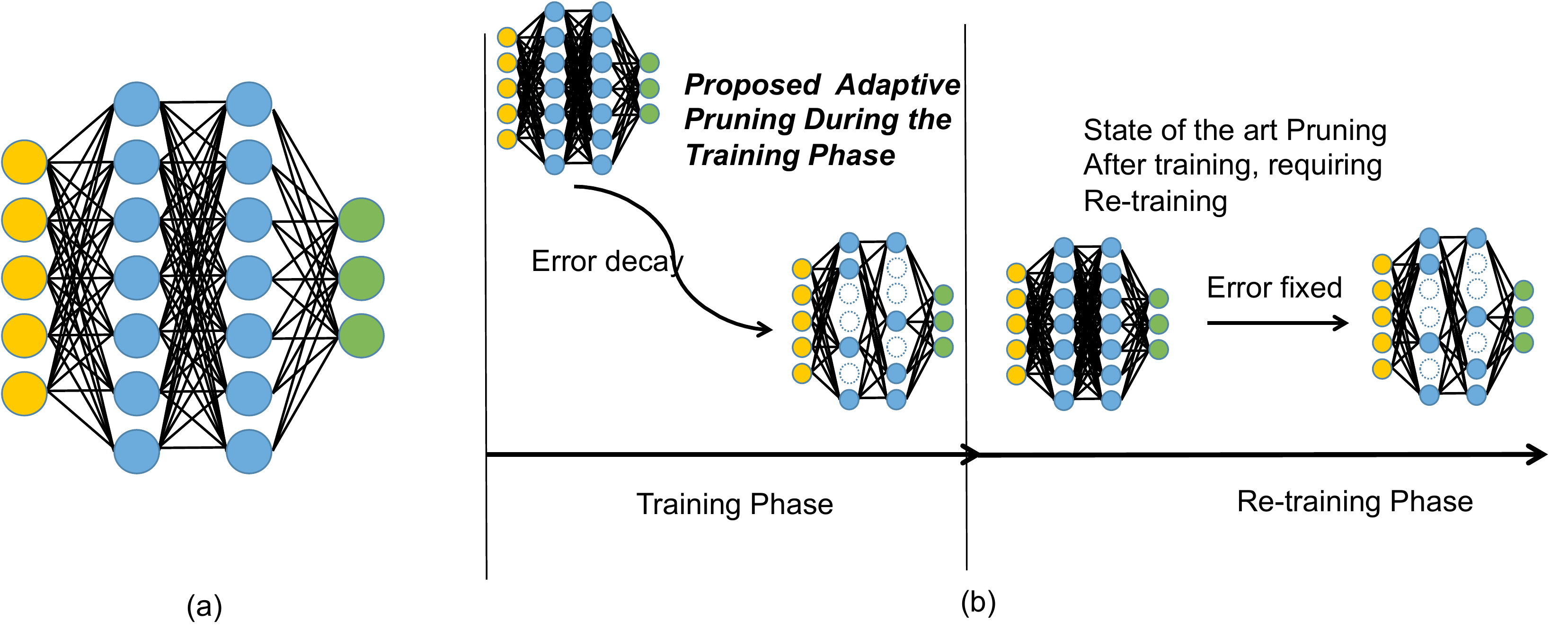}
\caption{(a) An example of a Deep Neural Network (DNN) with two hidden
layers (color coded - blue) (b) Comparison of our approach (adaptive
apoptosis during the training phase) to state of the art approaches,
which conduct removal of non-contributing synapses after training is
over. {\bf Unlike state of the art approaches, the proposed approach does not require re-training} }
\label{fig:example_dnn}
				\centering
		\end{minipage}
\end{figure*}

A possible solution to this problem is to prune the neural network.
Usually, this is conducted after the training is completed by removing
{\em unnecessary} weights (and possibly)
neurons~\cite{yang:iccv15,han:nips15,squeezenet}.  After pruning, the
network is re-trained to stabilize the parameters~\cite{han:nips15} (The
scenario is shown on the right of Figure~\ref{fig:example_dnn}(b)).
Usually pruning implies that unnecessary parameters (and in a few cases
neurons) may be removed, without accuracy loss.

However, there are several shortcomings of existing approaches: 1)
Re-training is a time-consuming process. As an example, Han {\em et
al.} report a slowdown of up to 2.5x when re-training after
pruning~\cite{han:nips15}.  
This is especially problematic for large datasets, where
increasing the training time is unattractive. 2) Another problem with
the current approaches is the lack of theoretical underpinnings for
network pruning. Usually, this is not required for state of the art
networks, because they use the final network as starting point for
re-training. 
After each sub-step of pruning and re-training, the accuracy is compared
against the reference network -- to ensure no empirical loss of accuracy.  
However, this approach is not always efficient, since this exploration
is is expensive in time due to several re-training steps.

Our objective in this paper is to remove unnecessary neurons (and hence
synapses) for Deep
Learning algorithms
by {\em adaptively}
removing redundant neurons -- neuron {\em apoptosis} -- during the
training phase itself, while achieving similar accuracy as achieved
without apoptosis using the original neural network. Unlike existing
approaches, our adaptive apoptosis approach reduces the overall training time --
which makes it a very attractive solution for reducing the
computational/storage requirements, while gaining speedup during the
training phase itself.

\subsection{Contributions:}
Specifically, we make the following contributions in this paper:

\begin{itemize}
		\item We propose novel heuristics to adaptively conduct neuron
				apoptosis during the training phase. We provide an
				in-depth discussion of the point of initial apoptosis,
				subsequent apoptosis and degree of apoptosis.
				Our heuristics rely on the intuition of
				the basic structure of the loss function (independent of the
				dataset), in addition to heuristics, which consider
				linear apoptosis.
		\item We provide theoretical underpinnings of our proposed
				solution. These are intended to bound the loss of accuracy incurred by
				neuron apoptosis and address challenges of considering
				input/output synapses to neuron types for apoptosis.
				Unlike existing literature -- where
				neuron pruning is executed after training to prevent
				accuracy loss -- this is a
				critical step to ensure the correctness of proposed
				heuristics. 
		\item We extend Caffe to use MPI (similar to FireCaffe~\cite{firecaffe}) -- so that the execution may be
				conducted on supercomputers, and other large scale
				systems (such as cloud computing systems). 
				With these
				extensions, our implementations is able to utilize large
				scale clusters using native implementations with
				multi-core systems and accelerators such as GPUs.
		\item We evaluate our proposed implementations with two clusters
				-- one connected with Intel Haswell and InfiniBand, and
				other connected with nVIDIA GPUs and InfiniBand. We use
				several large datasets for evaluation using multiple
				nodes on each cluster. Our evaluation indicates a
				reduction of up to 30x in overall parameters, a speedup
				of 2-3x -- unlike existing approaches which cause
				slowdown.
\end{itemize}
A very important science contribution of our paper is the improvement in
classification accuracy of Higgs Boson Dataset (represented by Receiver Operating Curve (ROC) -
Area Under Curve (AUC))
published in the literature by Sadowski {\em et
al.}~\cite{sadowski:nips14}. We are able to improve the AUC by 6
percentage points (from 0.88/1 to 0.94/1), while reducing the number of
parameters by {\bf 3x} and obtaining a speedup of {\bf 2.44x} in
comparison to the default (no apoptosis) algorithm.
Another important  artifact of our approach is the reduction in space and
computational requirements of the neural networks (up to {\bf 30x}), which can be realized
without incurring a penalty in training time.

The rest of the paper is organized as follows: In
section~\ref{sec:related}, we present related work on neural network
topologies. In section~\ref{sec:fundamentals}, we provide a brief
introduction to neural networks and Google TensorFlow. In
section~\ref{sec:apopspace}, we present a solution space to the problem of
neuron apoptosis, possible design choices, heuristics and perceived
benefits.  We also present theoretical underpinnings of our proposed
solutions (section~\ref{sec:proof}), and provide a proof on bounds to
accuracy loss. We present detailed performance evaluation in
section~\ref{sec:exp} and present conclusions in
section~\ref{sec:conclusions}.

\section{Related Work}
\label{sec:related}
We split the related work section on Deep Learning
algorithms and implementations in research on large scale systems and
pruning/compression algorithms.

\subsection{Systems (Multi-core/Many-core/Large Scale) Research}
The most widely used algorithm for training Deep Learning algorithms is
batch gradient descent. Several implementations of batch gradient
descent methods are available for sequential, multi-core and many-core
systems such as GPUs. The most prominent implementations are
Caffe~\cite{jia2014caffe} (GPUs), Warp-CTC (GPUs),
Theano~\cite{Bastien-Theano-2012, bergstra+al:2010-scipy} (CPUs/GPUs),
Torch~\cite{Collobert02torch:a} (CPUs/GPUs), 
and Google
TensorFlow~\cite{tensorflow2015-whitepaper} which uses nVIDIA CUDA Deep
Neural Network (cuDNN) and a multi-threaded implementation of batch gradient descent methods. 

Caffe has emerged as one of the leading Deep
Learning software, which can be used for developing
novel extensions, such as ones proposed in this paper. Caffe supports
execution on single node (connected with several GPUs) and recent
extensions include support on Intel systems. While we conduct the
proposed research with Caffe, the proposed extensions can also be
applied with TensorFlow.

Classical neural networks were shallow (1-2 layers), where batch
gradient descent methods worked well. However, with deeper networks, the
algorithms frequently suffered from the {\em vanishing gradient}
problem~\cite{Bianchini2014}).  The standard algorithm for training them
(described in section \ref{sec:fundamentals}) fails because the
gradients become smaller by several orders of magnitude as the network
becomes deeper.  This problem was solved in~\cite{Hinton06afast}
and~\cite{NIPS2006_3048}, who demonstrated that a network can be trained
one layer at a time with
\textit{autoencoders}~\cite{HintonSalakhutdinov2006b}, and then put it
together into a single network for
classification~\cite{Vincent:2010:SDA:1756006.1953039}.  Another
solution, that we will use, are rectified linear units, which have
become the standard in the field, though our results are valid for other
types of neurons as well.
These optimizations are available in Caffe and other Deep Learning
packages.

\subsection{Neural Network Pruning/Compression}
Network pruning is typically considered for reducing the memory
and computational requirements for execution on embedded devices.
Compression algorithms are then applied on these pruned networks for
realizing further memory savings.

In biology, apoptosis -- the death of neurons -- and its inverse neurogenesis,
have been studied~\cite{chambers2004simulated}, and it has been determined that
these processes can aid in learning effectively.  Chambers {\em et al.} modeled the
apoptosis/neurogenesis process by re-initializing the weights of
randomly selected neurons periodically. Their study concluded that
periodic apoptosis can improve the performance of a network. This forms
the basis of our paper -- by conducting adaptive apoptosis during the
training phase. 

Several researchers have conducted {\em offline} neuron apoptosis --
after the completion of the training phase. This is in sharp contrast to
our proposed approach, in which apoptosis is conducted during the
training phase itself, which reduces the overall training time, while
reducing the space requirements.
For offline apoptosis, researchers have conducted neuron apoptosis due to a lack of
computational resources~\cite{sietsma1988neural}, 
regularization to prevent overfitting~\cite{Reed:1993:PAS:2325855.2328312},
which provides algorithms for
removing synapses and possibly neurons.  
Kamruzzman {\em et al.} have demonstrated similar
results~\cite{DBLP:journals/corr/abs-1009-4983}, where the fundamental
objective is to remove redundant weights, and possibly neurons. The
pruning has been also applied for generating human-usable rules for
classification~\cite{gorban1999generation}. 
Other researchers have considered temporary neuron pruning --
typically referred as {\em Dropout}, proposed by Hinton {\em et al.}~
\cite{hinton2012improving}. A selective Dropout is proposed by other
authors~\cite{goodrich2014neuron}. However, since the Dropout is
temporary, they do not affect the final topology of the neural network
-- although they help with regularization.

Recently, Han {\em et al.} have proposed methods to remove
non-contributing weights after the training phase~\cite{han:nips15}.
However, this incurs significant slowdown (up to 2.5x). Other approaches
-- such as HashedNets -- compress the neural networks without removing
weights/neurons~\cite{hashednets}.  Murray and Chiang have proposed 
methods to remove non-contributing synapses (which are significantly
different from our approach of removing redundant neurons)~\cite{murray2015auto, hu2016network}.
Since we are primarily focused on removing redundant neurons adaptively, we consider our approaches to be complimentary to their approach. 

\section{Fundamentals}
\label{sec:fundamentals}

\subsection{Neural Networks}
Neural Networks are a  class of Machine Learning algorithms, which 
emulate the computational structure of the brain to learn nonlinear functions. 
The basic unit of a neural network is a \textit{neuron}, and neurons are
interconnected using {\em synapses}.

\subsubsection{Activation Functions}

\begin{figure*}[htbp]
\begin{minipage}[t]{0.95\columnwidth}
\centering
\includegraphics[width=\columnwidth]{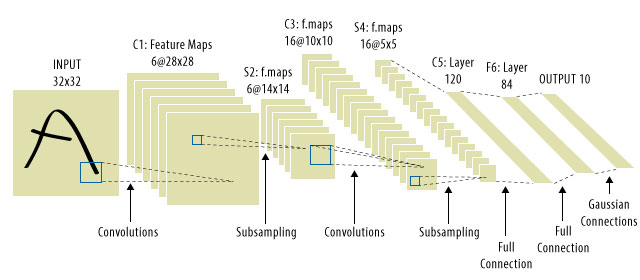}
\caption{An example of CNN execution on letter ``A'' from MNIST with  LeNet-5~\cite{lecun1998gradient}}
\label{fig:CNNStruc}
\centering
\end{minipage}
\centering
\begin{minipage}[t]{0.95\columnwidth}
\centering
\includegraphics[width=\columnwidth]{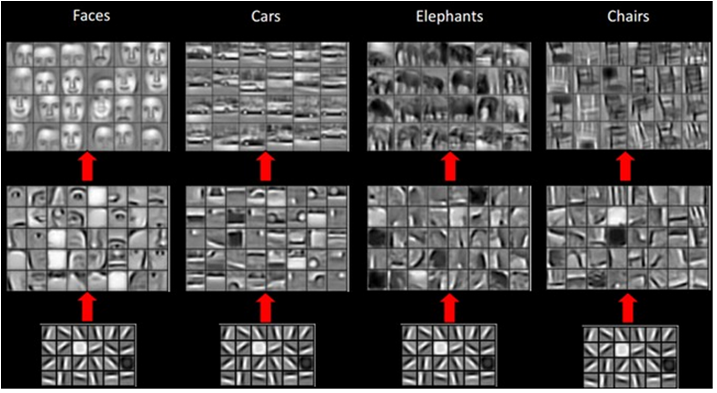}
\caption{Pictorial representation of features in 3 layers of a CNN~\cite{lee2009convolutional}. Notice the increasing complexity of features -- although first level features are the same}
\label{fig:CNNFeat}
\centering
\end{minipage}
\end{figure*}

There are several common nonlinear activation functions for neural
networks. In this paper, we will specifically focus on the two most
widely used, rectified linear units (ReLU) and sigmoid units.

\begin{eqnarray}
\mathrm{sigmoid}(x) & = & \frac{1}{1+\exp(-x)}\\
\mathrm{ReLU}(x) & = & \max(0, x)
\end{eqnarray}

\subsubsection{Convolutional Neural Networks}
Convolutional Neural Networks (CNN) are a widely used type of neural
network, which are specifically designed to preserve structure in the
data, such as the sequence of sounds in speech data or the relative
positions of pixels and features in an image.  

The fundamental unit of computation in CNN is a convolution -- which are
arrays in some dimension -- unlike vectors in DNN. Each neuron in a
convolution layer considers input from a small window (such as
\texttt{3x3,5x5}) in an image, applies a convolution and computes a
value, which is an indirect representation of the confidence of the
feature detected by the neuron. The window based computational structure
is useful for structured datasets, which can use a cascade of
convolutional layers to incrementally generate more complex features. 
An example of CNN is shown in Figure~\ref{fig:CNNStruc}, and
Figure~\ref{fig:CNNFeat} represents the features learned by a CNN using a
cascade of convolution layers. 
Besides convolution layers, a neural network may also consist of {\em
pooling layers}, which are also used for alleviating overfitting to the
presence of a feature in an image.

\subsection{Caffe}
Caffe~\cite{jia2014caffe} is a popular software package which provides abstractions for building neural networks of a wide range of topologies and training them with a wide range of optimizers.  Caffe provides abstractions of operations on tensors (multi-dimensional arrays), which are used for implementing Deep Learning algorithms.  Caffe builds a computational graph which consists of an input tensor followed by tensors for each individual hidden layer and output.  We choose Caffe because it is heavily optimized, and can be modified effectively through both the C++ backend and a Python interface.

Caffe's runtime is implemented using C++ -- which makes it attractive for extracting native performance.  We have modified this code for distributed memory implementation on large scale systems, using MPI to natively use network hardware and obtain optimal performance.  This is similar to FireCaffe~\cite{iandola2015firecaffe}, another distributed memory implementation of Caffe.  Additionally, Caffe abstracts GPU computations by leveraging nVIDIA CUDA Deep Neural Network Library (cuDNN).  As a result, the implementations are able to use large scale systems on traditional multi-core systems and many-core systems connected with GPUs.

\section{Neuron Apoptosis Solution Space}
\label{sec:apopspace}
In this section, we present a solution space for neuron apoptosis.
There are several important design considerations: the point of initial
apoptosis, when to conduct the subsequent apoptosis, apoptosis
termination and degree of apoptosis. We present design considerations
for each of these topics.

\subsection{Point of Initial Apoptosis}
The intuition behind initial apoptosis comes from the mammalian brain.  A significant apoptosis at any stage may have drastic consequences -- especially one very early or late in life.  However, the mammalian brain loses neurons periodically while retaining the fidelity of previously learned models (such as object, taste and voice recognition).

\subsubsection{Quarter-Life}
We take inspiration from particle physics to consider {\em half-life} to
be the point of initial apoptosis. It is possible to calculate half-life
statically (by using the number of epochs/training time provided by the
user -- generally as an input). 
However, with initial point as half-life, we would possibly be
conducting apoptosis at end of life. Hence, we instead consider the
point of initial apoptosis to be {\em quarter-life}.

\subsubsection{Random}
It is also possible to consider {\em random} time-stamp/epoch during training as the
initial point of apoptosis. In most cases, if the initial apoptosis
occurs {\em early}, it would lead to significant pruning with
potentially significant damage to the model accuracy. In other cases
(such as quarter-life/half-life/end of training), it would generate a high fidelity
model, albeit without observing speedup in training time.

\subsubsection{End of Training}
Previously proposed approaches perform apoptosis at the end of training.  These techniques remove synapses that fail to contribute, rather than redundant neurons.  However, this approach requires a re-training phase, as in Han {\em et al.}~\cite{han:nips15}.  To ensure that training time does not increase, we do not pursue this heuristic.

\subsection{Subsequent Apoptosis}
Another important design consideration is when to conduct subsequent
apoptosis. We present intuition behind our design choices here:

\subsubsection{Fixed/Random} 
An intuitive heuristic is to conduct subsequent apoptosis at
fixed/random intervals.  While there are several advantages to
fixed/random apoptosis, a high frequency would result in many distance
calculations, most of which will not result in any significant apoptosis
(as the weights would not change dramatically).  On the other hand, a
low frequency would likely conduct less apoptosis of neurons, but would
allow for a more significant change in weights when compared to a high
frequency.

\subsubsection{Logarithmic} We draw inspiration from the properties of
loss functions in Deep Learning algorithms for making a case of
logarithmic number of apoptoses.
In general, the decay of the loss function can be well represented using an exponential decay function.
Hence, it is expected that the subsequent apoptosis at
logarithmic steps will handle {\em half-life} of errors. Additionally,
with logarithmic number of apoptosis, the overall time spent in distance
calculations would also be minimized.

\subsection{Temporal Degree of Apoptosis}
Another design consideration is the temporal degree of apoptosis. 
Essentially, it is important to consider whether apoptosis should remain
fixed, become increasingly aggressive or conservative:

\subsubsection{Fixed}
The default choice is to use a fixed degree of apoptosis. The
expectation with this approach is to ensure that rate of redundant
neuron pruning is fixed over time.

\subsubsection{Increasingly Aggressive}
The intuition behind increasingly aggressive apoptosis is that the overall
possibility to prune redundant neurons is diminished -- as the
termination criteria is reached. The rate at which degree increases
itself has several choices (such as fixed-linear, random and others). For
simplicity, we propose to increase the degree of apoptosis linearly --
although without loss of generality, other functions may be applied as
well. 

\subsubsection{Increasingly Conservative}
On the contrary, the intuition behind conservative apoptosis is to save
the remaining neurons (after initial apoptosis) as much as possible,
especially if the early apoptosis already prunes the majority of redundant neurons. 
Without loss of generality, we decrease the degree of apoptosis with
number of apoptosis steps.

\subsection{Modeling Space-Time Complexity}
An important aspect of apoptosis is to ensure that the overhead of the
proposed heuristics is negligible in terms of space and time complexity. 
For each of the proposed heuristics earlier, the space complexity is
$O(1)$, since we only need to save a few scalars (such as parameters to
a linear function for degree of apoptosis). Hence, the space overhead of
our approach is fairly limited.

At each apoptosis step, we need to prune the redundant neurons in each
layer. Let $n_i$ represent the number of neurons in the $i_{th}$ layer.
Hence we need to conduct ${n_i}^2$ calculations at each layer (for each
inter-neuron distance calculation). This is particularly scalable in for
multiple reasons: 1) The distance calculation is computed sporadically
(such as only $\log$ number of iterations) 2) With each apoptosis, is it
expected that a significant number of redundant neurons are removed
(${\hat{n}_i} \ll n_i$), hence the subsequent distance calculations are
relatively insignificant -- and easily amortized over the cost of
overall training and benefits received by adaptive apoptosis during the
training phase.

\subsection{Other Design Considerations}

\subsubsection{Incoming vs. Outgoing Synapses}
Every neuron has both incoming and outgoing synapses, and they play different
roles (an example is shown in Figure~\ref{fig:apop}).  In sigmoid neurons, both sets of synapses can be used for apoptosis.
With $\ReLU$ neurons, however, only the incoming synapses can be used.
This is because no matter how similar the outgoing synapses, if the incoming ones
are very different, then for a large range of data, one of the neurons will output zero
and the other will output an arbitrarily large positive number.  In section \ref{sec:proof},
we give proofs of this restriction and on the amount of error introduced by these
operations.
\begin{figure}[htbp]
\centering
\includegraphics[width=0.5\columnwidth]{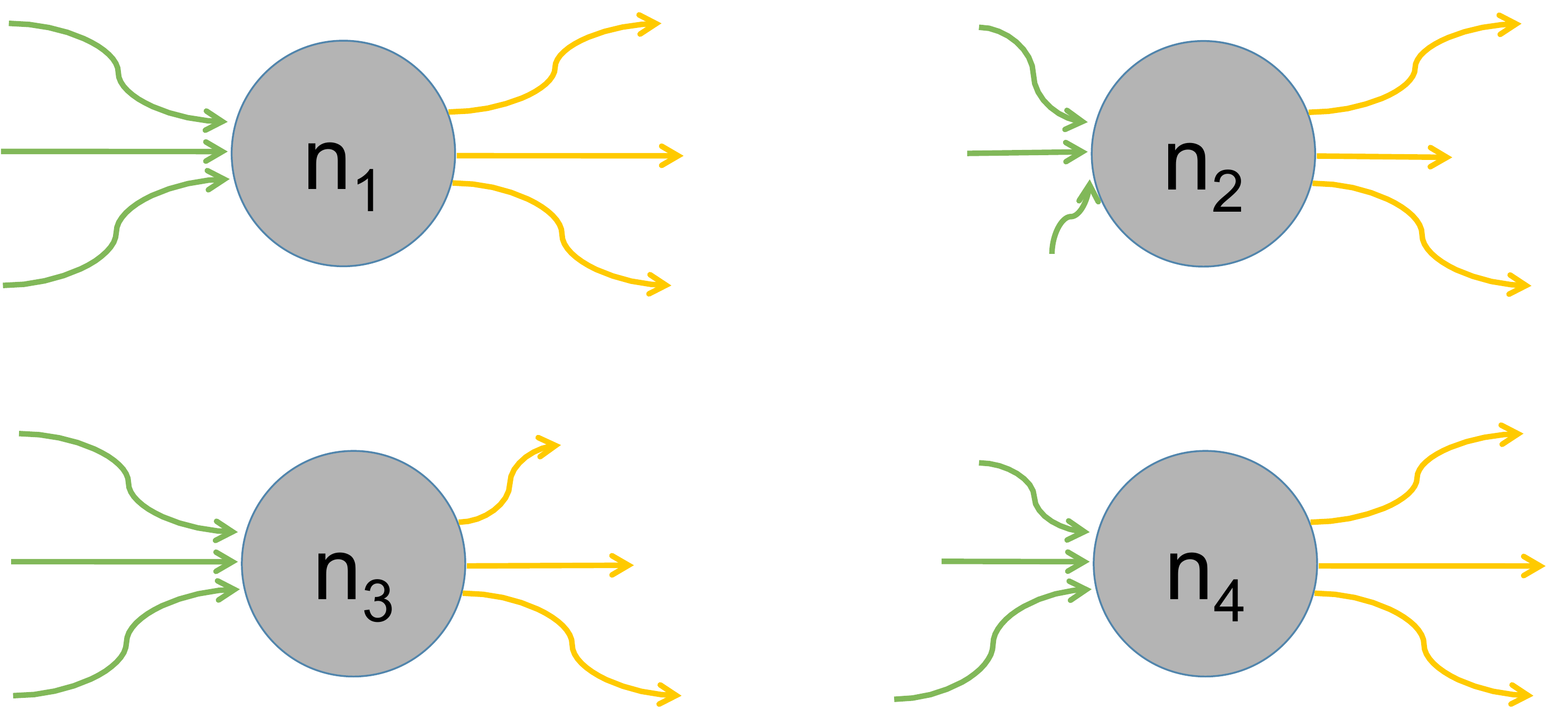}
\caption{Neurons $n_1$ and $n_3$ have similar incoming synapses and $n_1$ and $n_4$ have similar outgoing synapses.  In section \ref{sec:proof} we show that in both sigmoid and $\ReLU$ we can remove $n_3$, but only in sigmoid can we remove $n_4$.}
\label{fig:apop}
\end{figure}

\section{Proofs for Selecting Input/Output Synapses and  Bounding Accuracy Loss}
\label{sec:proof}

For any given neuron $n$, let $v$ denote the \textit{incoming weights}
which affect the argument of the function, and let $w$ denote the
\textit{outgoing weights}, which are the coefficients applied to the
value of the function before being input into the next layer of neurons.
We will use $\cdot$ to represent dot product and $\odot$ to represent component-wise
multiplication of a vector with a scalar or another vector.

\subsection{Selecting Input/Output Synapses}
\begin{proposition}
Let $n_1$ and $n_2$ be sigmoid neurons, $v_1$ and $v_2$ their incoming
weights and $w_1$ and $w_2$ their outgoing weights.  Then
\begin{equation}
\label{approxeq}
w_1\odot \sigma(v_1\cdot x)+w_2\odot\sigma(v_2\cdot x) \approx w\odot \sigma(v\cdot x)
\end{equation}
holds for all $x$ if either
\begin{enumerate}
  \item $v_1\approx v_2$.  Then $w=w_1+w_2$ and $v=v_1$.
	\item $w_1\approx \alpha w_2$.  Then $w=w_1+w_2$ and $v=\frac{\alpha v_1+v_2}{\alpha+1}$ so long as $\alpha\neq -1$.
\end{enumerate}
\end{proposition}

\begin{proof}
Equation \ref{approxeq} is a system of equations in $v$ and
$w$.  If we take a linear approximation of $\sigma$, which is given by
$\sigma(x)=\frac{1}{2}+\frac{x}{4}+O(x^3)$, we transform it into a
linear system of equations.  Equating constants implies that
$w=w_1+w_2$ (and note that this also means that far from the origin, the
approximation error will be very small, because sigmoid has horizontal asymptotes.)
Hence:

\begin{equation}
w_1\odot v_1\cdot x+ w_2\odot v_1\cdot x = (w_1+w_2)\odot v\cdot x
\end{equation}

We can solve this for each component of the vector $v$, which gives us
\begin{equation}
v_k = \frac{w_{1,j} v_{1,k}+ w_{2, j} v_{2, k}}{w_{1, j}+ w_{2, j}}
\end{equation}
which must hold for all the components $w_{i, j}$ of the vectors $w_1$
and $w_2$, simultaneously.

Equation \ref{approxeq} is only possible if the $w$'s cancel or if both of them have all
entries the same, which is unlikely.  The two conditions in the
statement, that $v_1\approx v_2$ or $w_1\approx \alpha w_2$, both result
in cancellations, and upon simplifying, we get the value for $v$ given.
\end{proof}

Below, we will compute a bound for the error in the case of similar
incoming weights ($v_1\approx v_2$).  The situation for outgoing weights
($w_1\approx w_2$)
is similar, but bounded by larger error in the worst case (proof not included due to space limitations).
In practice, though, this bound is not a very tight one.
In section \ref{sec:exp}, we will use both forms of apoptosis for
sigmoid neurons, though we remark that this larger error bound does
become a problem when attempting to do apoptosis on input features.

The situation for ReLU neurons is simpler.  No single ReLU neuron can
closely approximate a general sum of two ReLU neurons, an an extreme
example of this failure is
\begin{equation}
\ReLU(x)+\ReLU(-x) = |x|.
\end{equation}

However, if the two ReLU neurons have incoming weight vectors pointing
in almost the same direction, so that one is approximately a positive
scalar multiple of the other, we have

\begin{proposition}
Let $n_1, n_2$ be ReLU neurons with incoming weights $v_1, v_2$ and outgoing weights $w_1,w_2$.  Then, if $v_2\approx \alpha v_1$ where $\alpha>0$, we have
\begin{eqnarray}
w_1\odot \ReLU(v_1\cdot x) + w_2 \odot \ReLU(v_2\cdot x) &\approx&\\
(w_1+\alpha w_2)\odot \ReLU(v_1\cdot x)&&
\end{eqnarray}
\end{proposition}

\begin{proof}
The statement follows from the approximate sequence of equations

\begin{eqnarray}
w_1\odot \ReLU(v_1\cdot x) + w_2 \odot \ReLU(v_2\cdot x) &=&\\
w_1\odot \max(0, v_1\cdot x) + w_2\odot\max(0, v_2\cdot x)&\approx&\\
w_1\odot\max(0, v_1\cdot x) + w_2\odot\max(0, \alpha v_1\cdot x)&=&\\
w_1\odot\max(0, v_1\cdot x) + w_2\odot\alpha\max(0, v_1\cdot x)&=&\\
(w_1 + \alpha w_2)\odot\max(0, v_1\cdot x)
\end{eqnarray}
where error is only introduced in the substitution step where $v_2$ is
replaced by $\alpha v_1$.
\end{proof}

This situation  replaces a pair of half-planes where the neurons are
active, 
with a
single half-plane of activity, and because the two half-planes are very
close together, this is a good approximation.

The above two propositions show that we can remove sigmoid neurons that
have incoming weights nearly identical to another sigmoid neuron.
We can also remove sigmoid neurons where outgoing weights are a multiple of the outgoing weights
of another neuron (except -1). ReLU neurons whose incoming weights
are a positive multiple of those of another may be removed as well, with minimal change to the
output function.

\subsection{Proofs on Error Bounds}

In each of these cases, we can understand error bounds for what happens
after apoptosis by looking at how different the initial and final output
are.

\begin{proposition}
\label{prop:cases}
Let $v_1, v_2$ be incoming weight vectors, and assume that $|v_2-\alpha v_1|<\epsilon$ for some $\alpha>0$. Then \[|\ReLU(v_1\cdot x)+\ReLU(v_2\cdot x)-(1+\alpha)\ReLU(v_1\cdot x)|\] is at most \[\max(\alpha|v_1||x|, |v_2||x|, |x|\epsilon).\]
\end{proposition}

\begin{proof}
The quantity in question is
\begin{eqnarray}
\ReLU(v_1\cdot x)+\ReLU(v_2\cdot x)-(1+\alpha)\ReLU(v_1\cdot x)&&\\
\max(0, v_1\cdot x) + \max(0, v_2\cdot x)-(1+\alpha)\max(0, v_1\cdot x)&&\\
\max(0, v_2\cdot x)-\max(0, \alpha v_1\cdot x)&&
\end{eqnarray}

Here, we encounter the difficulty of the shape of the $\ReLU$ neuron, and have to do a case study.  There are four possibilities:

\begin{enumerate}
  \item if $v_2\cdot x\leq 0$ and $v_1\cdot x \leq 0$, then both are zero so the difference is $0$.
	\item if $v_2\cdot x\leq 0$ and $v_1\cdot x \geq 0$, then $\max(0, v_2\cdot x)=0$ so the difference is $|\alpha v_1\cdot x|\leq \alpha|v_1||x|$
	\item if $v_2\cdot x\geq 0$ and $v_1\cdot x \leq 0$, then $\max(0, v_1\cdot x)=0$ so the difference is $|v_2\cdot x|\leq |v_2||x|$
	\item if $v_2\cdot x\geq 0$ and $v_1\cdot x \geq 0$, then the difference is $|v_2\cdot x-\alpha v_1\cdot x|=|(v_2-\alpha v_1)\cdot x|$, which is $|v_2-\alpha v_1||x||\cos\theta|$.  This is then at most $|x|\epsilon$.
\end{enumerate}
\end{proof}

The dependence on $|x|$ in this is a manifestation of the fact that
$\ReLU$ does not saturate, and can grow without bound.  The worst error
regions, though, are where only one of $v_1\cdot x$ and $v_2\cdot x$ are
positive, but these regions are small (measured in angle) because $v_2$
closely approximates $\alpha v_1$, so they are nearly in the same
direction in the first place.  The fact that the error can be without
bound for arbitrary data does not arise for sigmoid, which saturates.

We note that the same logic as in case 4 of Proposition~\ref{prop:cases} implies that if $|v_1-v_2| < \epsilon$, then
\begin{equation}
\label{eqn:lem4}
|v_1\cdot x-v_2\cdot x|< |x|\epsilon
\end{equation}

\begin{proposition}
Let $|v_1-v_2|<\epsilon$.  Then 
\begin{equation}
|w_1\odot \sigma(v_1\cdot x)+w_2\odot\sigma(v_2\cdot x)-(w_1+w_2)\odot \sigma(v_1\cdot x)|
\end{equation}
is at most
\begin{equation}
\frac{|x||w_2|\max\left(|\exp(v_1\cdot x)-\exp(v_1\cdot x\pm \epsilon|x|)|\right)}{(\exp(v_1\cdot x)+1)(\exp(v_2\cdot x)+1)}
\end{equation}
\end{proposition}

We note that the factors including $|x|$ in the numerator are dominated by the ones in the denominator, so as expected, for large inputs, the horizontal asymptotes of sigmoid cause the error to be very small.

\begin{proof}
We proceed in stages.  First, we note that the terms involving $w_1$
cancel, giving us $|w_2\odot\sigma(v_2\cdot x)-w_2\odot\sigma(v_1\cdot
x)|$.  We can pull out the $w_2$, to get the quantity
$|w_2||\sigma(v_2\cdot x)-\sigma(v_1\cdot x)|$.

As $|v_1-v_2|<\epsilon$, equation \ref{eqn:lem4} tells us that $|v_2\cdot
x-v_1\cdot x|<|x|\epsilon$.  So, if $a=v_2\cdot x$ and $b=v_1\cdot x$,
we must only determine the error in $|\sigma(a)-\sigma(b)|$ for
$|a-b|<\epsilon|x|$.  But
$|\sigma(a)-\sigma(b)|=\frac{|\exp(a)-\exp(b)|}{(\exp(a)+1)(\exp(b)+1)}$,
and $|\exp(a)-\exp(b)|$ must at most be $\max(|\exp(b)-\exp(b\pm
|x|\epsilon)|)$, because $a$ and $b$ are close together, which gives the
result.  \end{proof}

Even for a fixed
$\epsilon$, this derivation may result in significant error due to the factor of
$|w_2|$, in practice $|w_2|$ will not be so much larger than $|v_1|$ and
$|v_2|$ that the exponentials in the denominator will not cause the
error to be small.

In practical use, however, fixing a single $\epsilon$ is suboptimal.
This is due in part to the vanishing gradient problem inherent in any
back-propagation trained neural network.  Because the neurons in
different layers will evolve in weight-space at different rates,
different values of $\epsilon$ would be needed for each layer.  An alternate
approach is to define a scaling \textit{factor}, $f$, and for each
neuron $n$ with incoming weights $v$ and outgoing weights $w$, instead
of looking for $n'$ with $v'$ and $w'$ such that $|v-v'|<\epsilon$ or
$|w-w'|<\epsilon$, we look for those with $|v-v'|<|v|/f$ or
$|w-w'|<|w|/f$.

\begin{figure}[htbp]
\centering
\includegraphics[width=0.5\columnwidth]{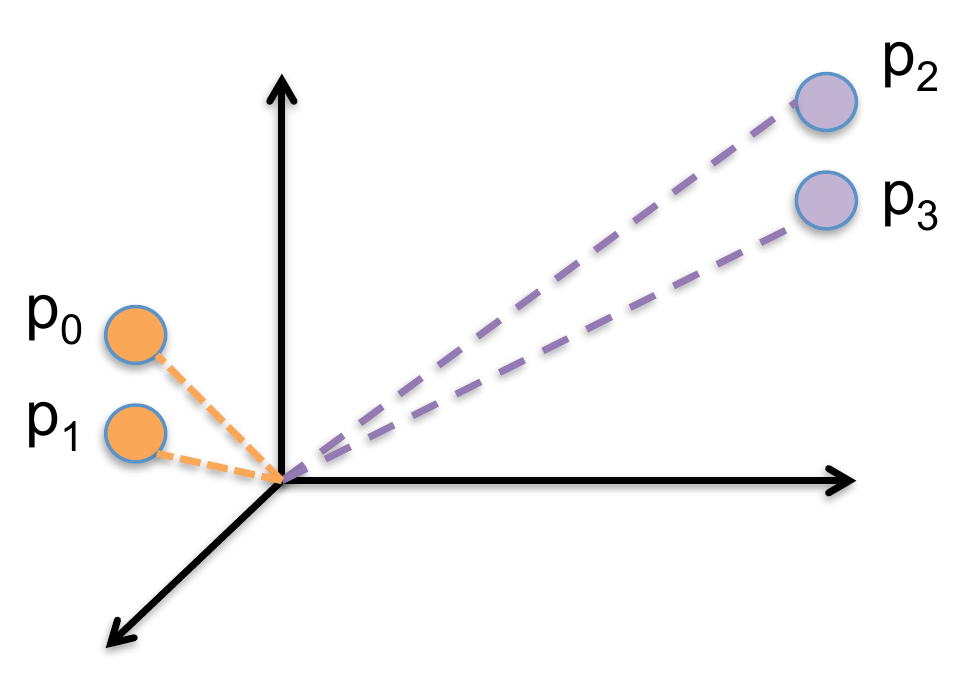}
\caption{Two pairs of neurons which are unit distance apart, but
neurons $p_2$ and $p_3$ are closer
}
\label{fig:circles}
\end{figure}

This handles several problems, most prominently the dimensionality
problem.  Randomly chosen points are expected to be farther
apart in higher dimensional spaces.  As they will also be farther from
the origin, this also allows the dimension to affect the range being
checked for apoptosis.

In Section \ref{sec:exp}, when we conduct performance evaluation, we will have
five levels of apoptosis.  The normal level will have angle $f=\nor$.
The others will be conservative, very conservative, aggressive and very
aggressive, with $f=\con, \vcon, \agg$ and $\vagg$ respectively.

\section{Large Scale Parallelization}
\label{sec:parallel}
In this section, we present the solution space in distributed memory
implementation -- including multi-core and many-core architectures. 
Our objective is to extract
the best possible performance -- while leveraging the accelerator based
systems (such as GPUs) and traditional multi-core architectures as well,
using multi-threading.

Our objective is also to leverage existing Deep Learning software --
such as Caffe/TensorFlow/Theano -- for our large scale
implementations of the proposed heuristics, so that existing
optimizations (such as Momentum, AdaGrad, Dropout) can be combined with our own
heuristics.  We specifically use Caffe, since it is an easily
extensible dataflow programming model, with already existing
implementations on GPUs using cuDNN and multi-threading implementations
for multi-core architectures.

\subsection{Caffe Runtime Changes}
Caffe sets up a \textit{solver} and a \textit{network}, the latter contains the data and the weights of the model whereas the former dictates the rules for performing gradient descent.  Our implementation is directly in the Caffe runtime, which speeds up performance significantly compared to implementation in the Python interface.

\subsection{Distributed Memory Parallelization}
We model our distributed memory implementation on the existing Caffe multi-GPU parallelization.  As such, our implementation is based on data parallelism, where the model is replicated and data is distributed, rather than model parallelism, where the model is distributed across multiple nodes.  We use MPI~\cite{2016arXiv160302339V,mpi1,mpi2}, which can use high performance interconnects, like InfiniBand, natively, making it suitable for use with supercomputers.

The solver has to be recreated after each apoptosis, but this is not a computationally expensive procedure.  Thanks to the use of the Python interface, the data does not need to be loaded every time, which leaves the memory requirements approximately the same as if one solver is used continuously, and cuts the time to recreate the solver down to a negligable factor compared to training time.

\section{Performance Evaluation}
\label{sec:exp}
In this section, we present a detailed evaluation of the proposed
heuristics on two InfiniBand clusters --  one connected with Intel
Haswell CPUs and other connected with nVIDIA Tesla K40m GPUs.

\subsection{Hardware and Software Details}

\subsubsection{Hardware}
Our Testbed consists of 27 compute nodes, where
each compute node has two sockets, and each socket is 10-core
Intel(R) Xeon(R) CPU E5-2680 v2 @ 2.80GHz. Each compute node is
connected with 750 GB of main memory, and InfiniBand QDR interconnect.
Six compute nodes are also connected with nVIDIA Tesla k40m GPUs. We
refer to our multi-core cluster as {\em CPU cluster} and other testbed
as {\em GPU cluster}.

\subsubsection{Software}
We use OpenMPI v1.8.3 for performance evaluation
with Intel compiler 16.0.1. For using GPUs, we use cuDNN v3, CUDA
v7.0.28, and a version of Caffe modified to work with MPI.

\subsection{Datasets}

We consider several well studied datasets for performance evaluation.
Our primary large datasets are Higgs Boson~\cite{sadowski:nips14}
classification dataset (11M samples) and well studied
ImageNet~\cite{imagenet} classification datasets ($\approx$1.3M images).
The evaluation also includes Handwritten Digits recognition (MNIST) to establish a baseline.

\subsection{MNIST}
MNIST is a well studied dataset in literature. We use MNIST as one of
the prominent datasets for studying the impact of a large combination of
heuristics proposed in section~\ref{sec:apopspace}. We also study the
scaling effect on CPU and GPU cluster. For scaling, we utilize two main
changes: 1) A single process per and several threads per node for better
memory utilization and lesser replication of model 2) Higher learning
rate (0.1) to mitigate the effect of averaging the weights across
compute nodes. 

Figures~\ref{fig:mnistperf} and~\ref{fig:mnistgpu} show the results for
CPU and GPU clusters respectively. In Figure~\ref{fig:mnistperf}, we
observe an overall speedup of {\bf 3.2x} in comparison to default
algorithm with Normal apoptosis -- with no loss of accuracy and {\bf
11x} reduction in parameters. The speedup and parameter reduction for
aggressive apoptosis is higher, but it leads to a loss of more than 5\%
accuracy -- so we do not consider it as a viable alternative. We observe
an accuracy of 96.5-97\% accuracy for each of these executions (except
aggressive) -- which matches with well published literature for DNN and
MNIST. For evaluation with CNN, we use our GPU cluster as shown in
Figure~\ref{fig:mnistgpu}. We use 6 GPUs and compare the speedup and
parameter reduction. We observe a 3x reduction in parameters for normal
apoptosis, accuracy $\approx$ 99.2\% -- which is state of the art for
CNN executions. 

{\em Using MNIST as a reference set, for rest of the evaluation, we use
quarter-life with logarithmic subsequent apoptosis, and fixed degree of
apoptosis. Unless specified otherwise, we use normal factor (1.75) for
apoptosis.}

\begin{figure}[htbp]
\centering
\includegraphics[width=\columnwidth]{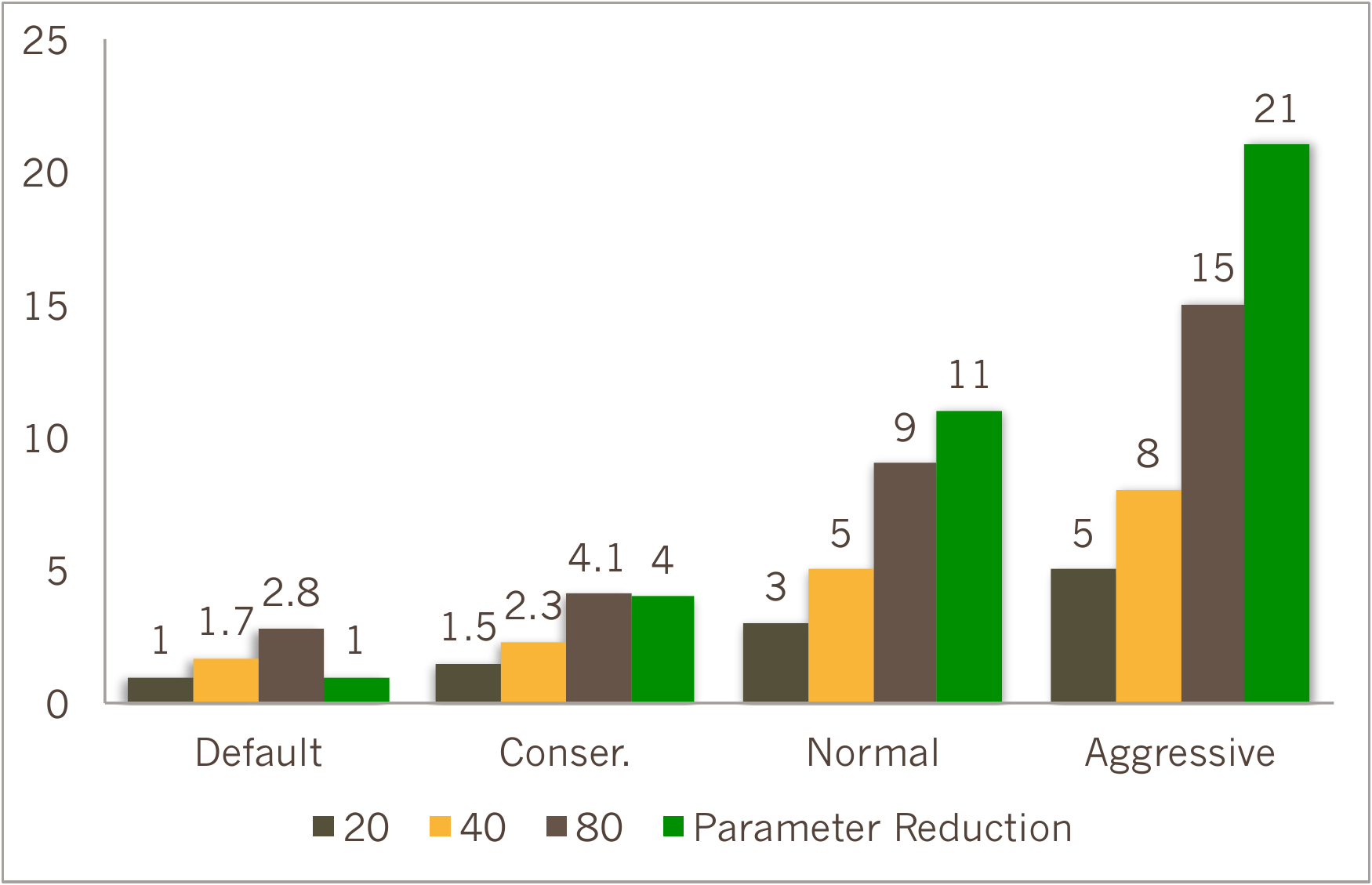}
\caption{Relative Speedup and Parameter Reduction 
to Default using DNN with \texttt{512x512} network with MNIST}
\label{fig:mnistperf}
\end{figure}
\begin{figure}[htbp]
\centering
\includegraphics[width=\columnwidth]{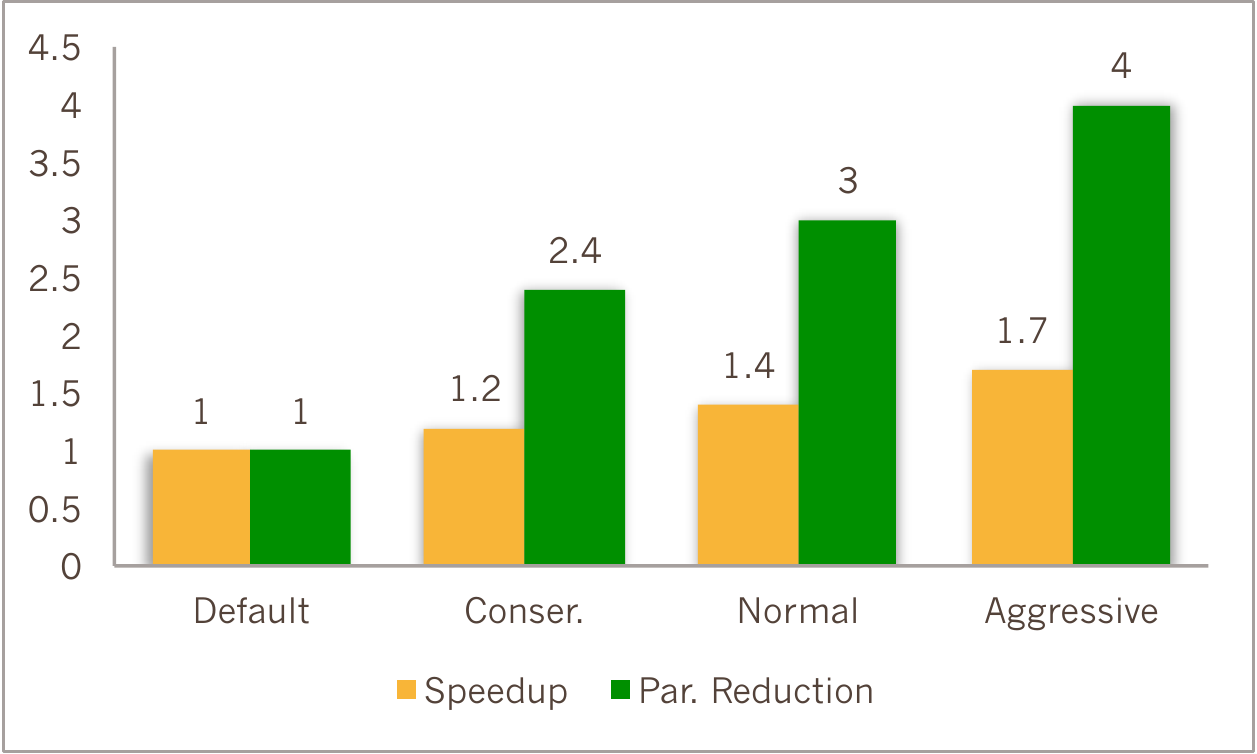}
\caption{Relative Speedup and Parameter Reduction using CNN with \texttt{32x64 Conv.
Layers and 256 Fully connected layer} network with MNIST and 6 GPUs}
\label{fig:mnistgpu}
\end{figure}

\subsection{Higgs Boson Particle Classification}
The Higgs Boson particle classification dataset is a critical dataset used for model generation
and discovery of exotic particles. 
Sadowski {\em
et al.} published and studied the dataset with Deep Learning algorithms
using a three layer DNN (\texttt {500x500x500}) network with neuron
dropout~\cite{sadowski:nips14}. Our objective with this workload is
two-fold: 1) Reduce the training time to learn the model, while
conducting neuron apoptosis 2) maintaining -- and possibly improving --
accuracy (measured using area under curve (AUC), which is the probability
that a randomly selected positive sample will be rated higher than
a randomly selected negative example), as suggested by
Sadowski {\em et al.}. For this purpose, we start with a bigger network
to possibly improve accuracy -- while utilizing apoptosis to remove
redundant neurons. We execute this dataset using up to 540 cores.

Figure~\ref{fig:higgs_speedup} shows the results. With aggressive
apoptosis, we are able to observe a speedup of {\bf 6x} -- while no
reduction in Area Under Curve (AUC). We match the AUC reported by
Sadowski {\em et al.}~\cite{sadowski:nips14}, while providing a huge
speedup.

\begin{figure}[htbp]
\centering
\includegraphics[width=\columnwidth]{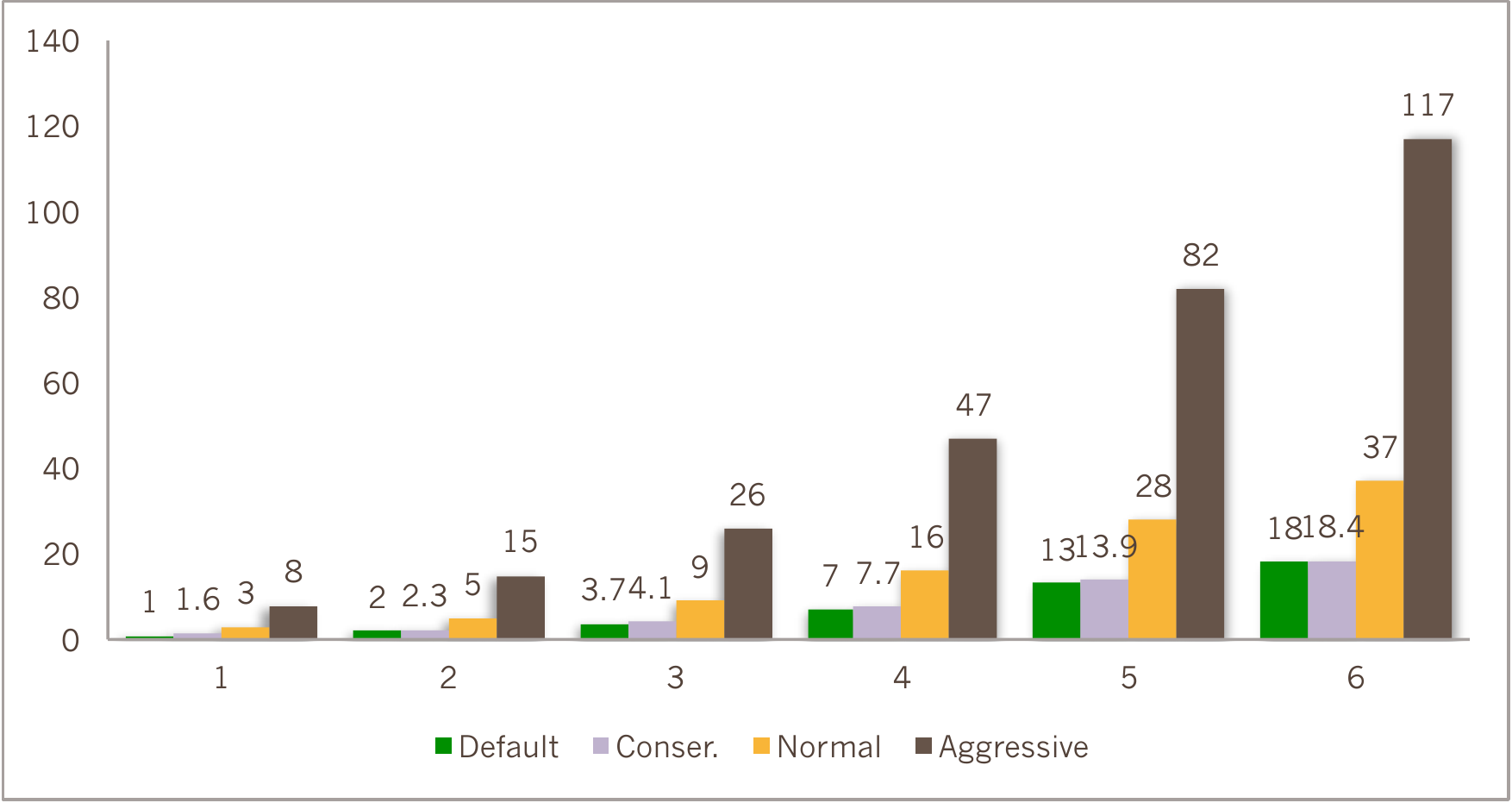}
\caption{Higgs Boson Relative Performance to 1 Node using up to 540
cores (One node has 20 cores) with \texttt{512x512x512} network
}
\label{fig:higgs_speedup} 
\end{figure}

\subsubsection{Science Results with Neuron Apoptosis}
The peak AUC reported by Sadowski {\em et al.} is 0.88/1. However, we
wanted to understand whether it is possible to create a neural network
topology -- which can provide better prediction for a scientist for
separating Higgs Boson particle from other particles. We used a deeper
network than ever considered before with Higgs by using a
\texttt{512x512x512x512} network. Being a 4 layer network, we used
autoencoders for learning weights layer by layer, conducted training on
the output layer and compared the
results for apoptosis and the default algorithm. 

Our evaluation indicates major contribution to the field: 1) Using
deeper networks, our AUC is {\bf 0.94/1} -- which is 6 percentage points better than
previously published result 2) With proposed apoptosis, we reduce number
of parameters by {\bf 3x} in comparison to the parameters in the best
known AUC for Higgs Boson particle classification -- while achieving a
{\bf speedup of 2.44x} in comparison to the no apoptosis algorithm.

\subsection{ImageNet - ILSVRC12 Results}
In this section, we present results for ImageNet Large Scale
Visualization Research Challenge (ILSVRC12). The objective of this
evaluation is to compare the accuracy reported by several neural
networks -- including DNN and CNN networks -- after a fixed training
time provided to each of the approaches. 
ILSVRC12 consists of 1.3M images, each of which is labeled in one of the
1000 categories. The ultimate objective of the challenge is to improve
the classification accuracy -- or reduce the time to solution to achieve
a given accuracy.
We primarily use GPU cluster for ImageNet evaluation -- especially since
cuDNN is heavily optimized in comparison to CPU based implementation. 

During this investigation, we discovered that, although the standard method of
training AlexNet~\cite{NIPS2012_4824} calls for 360000 iterations through the dataset,
if we use an exponentially decaying learning rate with multiplicative factor 0.999964, the
model converges after only 60000 iterations.  We used this ``quick'' AlexNet solver in the
evaluation below, along with optimization by using as many as 8 GPUs.

For each non-AlexNet execution, we provide a training time of eight hours using 6
GPUs. We observed the following points: 1) DNN provided significant
reduction in parameters, with negligible accuracy loss, while providing
speedup 2) With CNN, the primary apoptosis is observed at the
conjunction of the convolution and fully connected layers. However, most
of the time is spent in convolutions, and hence the speedup is lesser
than the DNN networks.

\begin{table}[t]
		\caption{ImagaNet Evaluation. First column shows the Deep
		Learning algorithm used; second shows the sizes of hidden
		layers; third the relative reduction in parameters; fourth
		change in accuracy and fifth the speedup}
		\label{imagenet-results}
		\vskip 0.15in
		\begin{center}
				\begin{small}
						\begin{sc}
								\begin{tabular}{lcccccr}
										\hline
										
										Algo & Network
						 & Pa. Red.
						 & Acc. (\%) & Speedup \\
						\hline
						DNN     & 2048,2048 & 27x & $ -0.7$ & 2.1x \\
						DNN     & 4096 & 34x & $-0.5$ & 2.3x
						 \\  
						CNN     & 64,128 and 2048 & 18x & $-0.3$ & 1.2x
						\\  
						\hline
				\end{tabular}
		\end{sc}
				\end{small}
		\end{center}
		\vskip -0.1in
\end{table} 

For AlexNet, we trained using our improved solver for various levels of aggressiveness of Apoptosis.  For this network, the results were mixed.  For most apoptosis factors, either virtually no ($\leq1\%$) of parameters were removed, or else the vast majority ($\geq 99.9\%$) were, and so obtained either no speedup or else an extreme loss of accuracy.  For factor $1.35$, however, $72.48\%$ of the parameters were removed, leading to a speedup of 2.6x/iteration.  However, while this reduced model recovered the accuracy from before apoptosis, but did not, during the life of the quick solver, converge to the usual test accuracy of $54\%$.

\subsection{Discussion}

There are a few important observation from the previous sections: 1) In
many cases, adaptive apoptosis results in significant reduction of
parameters and improvement in training time 2) In case of Higgs Boson
dataset, we observed that deeper networks -- with lesser parameters than
the original network -- produced better results. We also observed a
conical shape with reducing neurons per layer (starting from input layer
and ending at the output layer) in the neural network topology after
apoptosis. This can be explained by the fact that number of features is
much higher than the number of classes. 

An important discussion is the guidance for future users of this
research. An advantage of the proposed methodology is that a user may
start with a very large network to build models for its dataset, and
expect that the apoptosis would remove redundant neurons -- while
gaining speedup, and possibly accuracy, especially if the stopping
criteria is fixed time/epochs. This is attractive for novice
and advanced users alike, as there is little guidance on
specification of neural network topology. The proposed approach
generates a much smaller, near optimal topology, which would be
sufficient for the user.

\section{Conclusions}
\label{sec:conclusions}
In this paper, we have presented novel techniques to adaptively remove
redundant neurons ({\em neuron apoptosis}) 
for accelerating Deep Learning algorithms (Deep
Neural Networks, Convolutional Neural Networks and Autoencoders) --
during the training phase itself on large scale systems. The proposed
techniques are in sharp contrast with existing
approaches~\cite{han:nips15}, which require
a {\em re-training phase} for removing the weights which do not
contribute -- resulting in a significant slowdown (2.5x reported by Han
{\em et al}~\cite{han:nips15}) of the training time. 

Our contributions in this paper include novel heuristics for deciding on
initial apoptosis, subsequent apoptosis and degree of apoptosis during
the training phase itself. We provide theoretical underpinnings to
study the apoptosis for several neuron types, consider apoptosis
for input/output synapses for several neuron types and provide proofs on bounding accuracy
loss with apoptosis. We implement our proposed heuristics using Caffe, by extending it to use MPI
and allow multi-node training. We evaluate our
proposed heuristics with several large datasets including Higgs Boson
particle dataset, ImageNet classification and MNIST
handwritten digit recognition. We use two clusters -- one connected with
Intel Haswell CPU and InfiniBand QDR and other connected with nVIDIA
GPUs and InfiniBand. Our evaluation indicates significant improvement in
multiple dimensions -- including a reduction in parameters by  30x,
 and 2-3x speedup in time
comparison to no apoptosis implementation. A major contribution of our
paper is also an improvement in classification accuracy for Higgs Boson
particle dataset, by using apoptosis on a Deeper network than published
architecture. We are able to improve the Area Under Curve (AUC), by 6
percentage points from best known result {\bf 0.88/1 to 0.94/1}, while
reducing the number of parameters by {\bf 3x} in comparison to best
result producing network~\cite{sadowski:nips14}  and achieving a {\bf
2.44x} speedup.

\bibliographystyle{IEEEtran}

\bibliography{apoptosis,vishnu}

\end{document}